\documentclass{article} %
\usepackage{nips12submit_e,times}

\usepackage{amsmath}
\usepackage{amsthm}
\usepackage{amssymb}
\usepackage[numbers]{natbib}
\usepackage{url}
\usepackage[pdftex]{graphicx}
\usepackage{color}
\usepackage{algorithm}
\usepackage[noend]{algorithmic}
\usepackage{xspace}

\makeatletter
\newtheorem*{rep@theorem}{\rep@title}
\newcommand{\newreptheorem}[2]{%
\newenvironment{rep#1}[1]{%
 \def\rep@title{#2 \ref{##1}}%
 \begin{rep@theorem}}%
 {\end{rep@theorem}}}
\makeatother

\newtheorem{theorem}{Theorem}
\newreptheorem{theorem}{Theorem}
\newtheorem{lemma}[theorem]{Lemma}
\newreptheorem{lemma}{Lemma}

\newcommand{\ALGa}{\text{\sc Reward-Doubling-1D}\xspace}
\newcommand{\ALGaguess}{\text{\sc Reward-Doubling-1D-Guess}\xspace}
\newcommand{\ALGb}{\text{\sc Reward-Doubling}\xspace}
\newcommand{\ALGs}{\text{\sc Smooth-Reward-Doubling}\xspace}

\newcommand{\supplementary}[1]{#1}

\nipsfinalcopy %

\renewcommand{\algorithmicensure}{}

\newcommand{\abs}[1]{\ensuremath{| #1 |}}
\newcommand {\bexp}[1]{\exp\left(#1\right)}

\newcommand {\floor}[1]{\left\lfloor#1\right\rfloor}
\newcommand{\eps}{\epsilon}
\newcommand{\grad}{\triangledown}
\newcommand{\ignore}[1]{}
\newcommand{\norm}[1]{\ensuremath{\| #1 \|}}

\newcommand {\p}[1]{\left(#1\right)}
\newcommand{\R}{\ensuremath{\mathbb{R}}}
\newcommand{\xs}{\mathring{x}}
\newcommand{\BO}{\mathcal{O}}
\newcommand{\eqr}[1]{Eq.~\eqref{eq:#1}}
\newcommand{\h}{\frac{1}{2}}

\newcommand{\alg}{\mathcal{A}}
\newcommand{\ti}{_{t+1}}  %
\newcommand{\Linf}{L_\infty}
\newcommand{\heps}{\tilde{\eps}}

\DeclareMathOperator{\Regret}{Regret}

\DeclareMathOperator{\Reward}{Reward}
\DeclareMathOperator{\Ei}{Ei}
\DeclareMathOperator*{\argmin}{arg\,min}

\DeclareMathOperator{\sign}{sign}

\DeclareMathOperator{\E}{E}
\DeclareMathOperator{\Prob}{Pr}

\definecolor{darkblue}{rgb}{0,0,0.7}

\title{
No-Regret Algorithms for Unconstrained\\Online Convex Optimization
}

\author{
Matthew Streeter \\
Duolingo, Inc.$^*$
 \\
Pittsburgh, PA 15232 \\
\texttt{matt@duolingo.com} \\
\And
H. Brendan McMahan \\
Google, Inc. \\
Seattle, WA 98103 \\
\texttt{mcmahan@google.com} \\
}

\begin{document}

\let\thefootnote\relax\footnotetext{$^*$This work was performed while the
author was at Google.}

\maketitle

\begin{abstract}
  Some of the most compelling applications of online convex
  optimization, including online prediction and classification, are
  unconstrained: the natural feasible set is $\R^n$.  Existing
  algorithms fail to achieve sub-linear regret in this setting unless
  constraints on the comparator point $\xs$ are known in advance.  We
  present algorithms that, without such prior knowledge, offer
  near-optimal regret bounds with respect to \emph {any} choice of
  $\xs$.  In particular, regret with respect to $\xs = 0$ is
  \emph{constant}.  We then prove lower bounds showing that our
  guarantees are near-optimal in this setting.
\end{abstract}

\section{Introduction}

Over the past several years, online convex optimization has emerged as
a fundamental tool for solving problems in machine learning (see,
e.g., \citep{cesabianchi06plg,shwartz12online} for an introduction).
The reduction from general online convex optimization to online linear
optimization means that simple and efficient (in memory and time)
algorithms can be used to tackle large-scale machine learning
problems.  The key theoretical techniques behind essentially all the
algorithms in this field are the use of a fixed or increasing strongly
convex regularizer (for gradient descent algorithms, this is
equivalent to a fixed or decreasing learning rate sequence).  In this
paper, we show that a fundamentally different type of algorithm can
offer significant advantages over these approaches.  Our algorithms
adjust their learning rates based not just on the number of rounds,
but also based on the sum of gradients seen so far.  This allows us to
start with small learning rates, but effectively increase the learning
rate if the problem instance warrants it.

This approach produces regret bounds of the form $\BO\big(R\sqrt{T}
\log ((1+R)T)\big)$, where $R = \norm{\xs}_2$ is the $L_2$ norm of
an arbitrary comparator.
Critically, our algorithms provide this
guarantee simultaneously for \emph{all} $\xs \in \R^n$, without any
need to know $R$ in advance.  A consequence of this is that we can
guarantee at most \emph{constant} regret with respect to the origin,
$\xs = 0$.
This technique can be applied to any online convex optimization
problem where a fixed feasible set is not an essential component of
the problem.  We discuss two applications of particular interest
below:

\paragraph{Online Prediction}
Perhaps the single most important application of online convex
optimization is the following prediction setting: the world presents
an attribute vector $a_t \in \R^n$; the prediction algorithm produces
a prediction $\sigma(a_t \cdot x_t)$, where $x_t \in \R^n$ represents
the model parameters, and $\sigma: \R \rightarrow Y$ maps the linear
prediction into the appropriate label space.  Then, the adversary
reveals the label $y_t \in Y$, and the prediction is penalized
according to a loss function $\ell: Y \times Y \rightarrow \R$.  For
appropriately chosen $\sigma$ and $\ell$, this becomes a problem of
online convex optimization against functions $f_t(x) =
\ell(\sigma(a_t\cdot x), y_t)$.  In this formulation, there are no
inherent restrictions on the model coefficients $x \in \R^n$.
The practitioner may have prior
knowledge that ``small'' model vectors are more likely than large
ones, but this is rarely best encoded as a feasible set $\mathcal{F}$,
which says:
``all $x_t \in \mathcal{F}$ are equally likely, and all other $x_t$
are ruled out.''  A more general strategy is to introduce a
fixed convex regularizer: $L_1$ and $L_2^2$ penalties are common, but
domain-specific choices are also possible.
While algorithms of this form have proved very effective at solving
these problems, theoretical guarantees usually require fixing a feasible
set of radius $R$, or at least an intelligent guess of the norm of an
optimal comparator $\xs$.

\paragraph{The Unconstrained Experts Problem and Portfolio Management}
In the classic problem of predicting with expert advice (e.g.,
\citep{cesabianchi06plg}), there are $n$ experts, and on each round $t$ the
player selects an expert (say $i$), and obtains reward $g_{t,i}$ from
a bounded interval (say $[-1, 1]$).  Typically, one uses an algorithm
that proposes a probability distribution $p_t$ on experts, so the
expected reward is $p_t \cdot g_t$.

Our algorithms apply to an unconstrained version of this problem:
there are still $n$ experts with payouts in $[-1,1]$, but rather than
selecting an individual expert, the player can place a ``bet'' of
$x_{t,i}$ on each expert $i$, and then receives reward $\sum_i x_{t,i}
g_{t,i} = x_t \cdot g_t$.  The bets are unconstrained (betting a
negative value corresponds to betting against the expert).  In this
setting, a natural goal is the following: place bets so as to achieve
as much reward as possible, subject to the constraint that total
losses are bounded by a constant (which can be set equal to some
starting budget which is to be invested).  Our algorithms can satisfy
constraints of this form because regret with respect to $\xs = 0$
(which equals total loss) is bounded by a constant.

It is useful to contrast our results in this setting to previous
applications of online convex optimization to portfolio management,
for example~\citep{hazan09investing} and~\citep{agarwal06portfolio}.
By applying algorithms for exp-concave loss functions, they obtain
log-wealth within $\BO(\log(T))$ of the best constant rebalanced
portfolio.  However, this approach requires a
``no-junk-bond'' assumption: on each round, for each investment, you
always retain at least an $\alpha > 0$ fraction of your initial
investment.  While this may be realistic (though not guaranteed!) for
blue-chip stocks, it certainly is not for bets on derivatives that can
lose all their value unless a particular event occurs (e.g., a stock
price crosses some threshold).  Our model allows us to handle such
investments: if we play $x_i > 0$, an outcome of $g_i = -1$
corresponds exactly to losing 100\% of that investment.  Our results
imply that if even one investment (out of exponentially many choices)
has significant returns, we will increase our wealth
exponentially.\footnote{Our bounds are not directly comparable to the
  bounds cited above: a $\BO(\log(T))$ regret bound on log-wealth
  implies wealth at least $\BO\big(\text{OPT}/T\big)$, whereas we
  guarantee wealth like $\BO\big(\text{OPT'} - \sqrt{T}\big)$.  But
  more importantly, the comparison classes are different.}  

\paragraph{Notation and Problem Statement}
For the algorithms considered in this paper, it will be more natural
to consider reward-maximization rather than loss-minimization.
Therefore, we consider online linear optimization where the goal is to
maximize cumulative reward given adversarially selected linear reward
functions $f_t(x) = g_t \cdot x$.  On each round $t = 1 \dots T$, the
algorithm selects a point $x_t \in \R^n$, receives reward $f_t(x_t) =
g_t \cdot x_t$, and observes $g_t$.  For simplicity, we assume
$g_{t,i} \in [-1, 1]$, that is, $\norm{g_t}_\infty \leq 1$.  If the
real problem is against convex loss functions $\ell_t(x)$, they can be
converted to our framework by taking $g_t = -\grad \ell_t(x_t)$ (see
pseudo-code for \ALGb), using the standard
reduction from online convex optimization to online linear
optimization~\citep{zinkevich03giga}.

We use the compressed summation notation $g_{1:t} = \sum_{s=1}^t g_s$
for both vectors and scalars.
We study the reward of our algorithms, and
their regret against a fixed comparator $\xs$:
\[ \Reward \equiv \sum_{t=1}^T g_t \cdot x_t
\qquad \text{and} \qquad
\Regret(\xs) \equiv g_{1:T} \cdot \xs - \sum_{t=1}^T g_t \cdot x_t.
\]

\paragraph{Comparison of Regret Bounds}

The primary contribution of this paper is to establish matching upper
and lower bounds for unconstrained online convex optimization
problems, using algorithms that require no prior information about the
comparator point $\xs$.
Specifically, we present an algorithm that, for any $\xs \in \R^n$,
guarantees $\Regret(\xs)\leq \BO\big(\norm{\xs}_2 \sqrt{T} \log
((1+\norm{\xs}_2) \sqrt{T})\big)$.  To obtain this guarantee, we show that
it is sufficient (and necessary) that reward is 
$\Omega(\exp(\abs{g_{1:T}}/\sqrt T))$ (see
Theorem~\ref{thm:rb}).  This shift of emphasis from
regret-minimization to reward-maximization eliminates the
quantification on $\xs$, and may be useful in other contexts.

Table~\ref{tab:bounds} compares the bounds for \ALGb (this paper) to
those of two previous algorithms: online gradient descent
\citep{zinkevich03giga} and projected exponentiated gradient descent
\citep{kivinen97exponentiated,shwartz12online}.  For each algorithm,
we consider a fixed choice of parameter settings and then look at how
regret changes as we vary the comparator point $\xs$.

Gradient descent is minimax-optimal \citep{abernethy08} when the
comparator point is contained in a hypershere whose radius is known in
advance ($\norm{\xs}_2 \le R$) and gradients are sparse ($\norm{g_t}_2
\le 1$, top table).  Exponentiated gradient descent excels when
gradients are dense ($\norm{g_t}_\infty \le 1$, bottom table) but the
comparator point is sparse ($\norm{\xs}_1 \le R$ for $R$ known in
advance).  In both these cases, the bounds for \ALGb match those of
the previous algorithms up to logarithmic factors, even when they are
tuned optimally with knowledge of $R$.

The advantage of \ALGb shows up when the guess of $R$ used to tune the
competing algorithms turns out to be wrong.  When $\xs = 0$,
\ALGb offers constant regret compared to
$\Omega(\sqrt T)$ for the other algorithms.  When $\xs$ can be
arbitrary, only \ALGb offers sub-linear regret
(and in fact its regret bound is optimal, as shown in
Theorem~\ref{thm:lower}).

In order to guarantee constant origin-regret, \ALGb frequently ``jumps''
back to playing the origin, which may be undesirable in some
applications.  In Section~\ref{sec:epochfree} we introduce \ALGs,
which achieves similar guarantees without resetting to the origin.

\newcommand{\T}{\rule{0pt}{2.2ex}}

\begin{table}
Assuming $\|g_t\|_2 \le 1$:\\
\begin{tabular}{|p{4.3cm}|c|c|c|}
\hline
\T %
& $\xs = 0$ & $\norm{\xs}_2 \le R$ & Arbitrary $\xs$ \\
\hline
\T Gradient Descent,
$\eta = \frac {R} {\sqrt T}$
&
$R \sqrt {T}$
&
$R \sqrt {T}$
& 
$\norm{\xs}_2 T$ \\
\ALGb%
&
$\eps$
&
$R \sqrt T \log\p{\frac {n (1+R) T} {\eps} }$
&
$\norm{\xs}_2 \sqrt T \log\p{\frac {n (1+\norm{\xs}_2) T} {\eps} }$
 \\
\hline
\end{tabular}

\vspace{0.1in}
Assuming $\|g_t\|_\infty \le 1$:\\
\begin{tabular}{|p{3.5cm}|c|c|c|}
\hline
\T %
& $\xs = 0$ & $\norm{\xs}_1 \le R$ & Arbitrary $\xs$ \\
\hline
\T Exponentiated G.D.
&
$R \sqrt {T \log n}$
&
$R \sqrt {T \log n}$
& 
$\norm{\xs}_1 T$ \\ 
\ALGb
&
$\eps$
&
$R \sqrt T \log\p{\frac {n (1+R) T} {\eps} }$
&
$\norm{\xs}_1 \sqrt T \log\p{\frac {n (1+\norm{\xs}_1) \sqrt T} {\eps} }$
 \\
\hline
\end{tabular}
\caption{Worst-case regret bounds for various algorithms
(up to constant factors).
  Exponentiated G.D. uses feasible set $\{x: \|x\|_1 \le R\}$, and \ALGb
  uses $\eps_i = \frac \eps n$ in both cases.}
\vspace{-0.15in}
\label{tab:bounds}  
\end{table}

\paragraph{Related Work}
Our work is related, at least in spirit, to the use of a momentum
term in stochastic gradient descent for back propagation in neural
networks~\citep{jacobs87rate,pearlmutter91momentum,leen93adaptive}.
These results are similar in motivation in that they effectively yield
a larger learning rate when many recent gradients point in the same
direction.

In Follow-The-Regularized-Leader terms, the exponentiated gradient
descent algorithm with unnormalized weights
of~\citet{kivinen97exponentiated} plays $ x\ti = \argmin_{x \in
  \R^n_+} g_{1:t} \cdot x + \frac{1}{\eta} (x \log x - x), $ which has
closed-form solution $x\ti = \exp(-\eta g_{1:t})$.  Like our
algorithm, this algorithm moves away from the origin exponentially
fast, but unlike our algorithm it can incur arbitrarily large regret
with respect to $\xs = 0$.  Theorem \ref {thm:ftrl-bad} shows that no
algorithm of this form can provide bounds like the ones proved in this
paper.

\citet{hazan08extract} give regret bounds in terms of the variance of the
$g_t$.  Letting $G = |g_{1:t}|$ and $H = \sum_{t=1}^T g_t^2$, they
prove regret bounds of the form $\BO(\sqrt{V})$ where $V = H - G^2/T$.
This result has some similarity to our work in that $G/\sqrt{T} = \sqrt{H -
V}$, and so if we hold $H$ constant, then when $V$ is low, the critical
ratio $G/\sqrt{T}$ that appears in our bounds is large.  However, they
consider the case of a known feasible set, and their algorithm
(gradient descent with a constant learning rate) cannot obtain
bounds of the form we prove.

\section{Reward and Regret} \label{sec:reward-regret}

In this section we present a general result that converts lower bounds
on reward into upper bounds on regret, for one-dimensional online
linear optimization.  
In the unconstrained setting, this result will be sufficient to
provide guarantees for general $n$-dimensional online convex
optimization.

\newcommand{\thmrbtxt}[2]{ Consider an algorithm for one-dimensional
    online linear optimization that, when run on a sequence of
    gradients $g_1, g_2, \ldots, g_T$, with $g_t \in [-1,1]$ for all
    $t$, guarantees
  \begin{equation} #1
      \Reward \geq 
     \kappa \exp\p{\gamma |g_{1:T}|} - \epsilon,
  \end{equation}
  where $\gamma, \kappa > 0$ and $\epsilon \geq 0$ are constants.
  Then, against any comparator $\xs \in [-R,R]$, we have 
  \begin{equation} #2
    \Regret(\xs) \leq \frac{R}{\gamma} 
  \p{ \log\p{\frac{R}{\kappa \gamma}} -1 }
  + \epsilon,
  \end{equation}
  letting $0 \log 0 = 0$ when $R=0$.  Further, any algorithm with
  the regret guarantee of \eqr{regbound} must guarantee the reward of
  \eqr{rewardbound}.  }
\begin{theorem}\label{thm:rb}
\thmrbtxt{\label{eq:rewardbound}}{\label{eq:regbound}}
\end{theorem}
We give a proof of this theorem in the appendix.  The duality between
reward and regret can also be seen as a consequence of the fact that
$\exp(x)$ and $y \log y - y$ are convex conjugates.
The $\gamma$ term typically contains a dependence on $T$ like
$1/\sqrt{T}$.  This bound holds for all $R$, and so for some small $R$
the $\log$ term becomes negative; however, for real algorithms the
$\eps$ term will ensure the regret bound remains positive.  The minus
one can of course be dropped to simplify the bound further.

\section{Gradient Descent with Increasing Learning Rates}
\label{sec:incgd}

In this section we show that allowing the learning rate of gradient
descent to sometimes increase leads to novel theoretical guarantees.

To build intuition, consider online linear optimization in one dimension,
with gradients $g_1, g_2, \ldots, g_T$, all in $[-1, 1]$.  In this
setting, the reward of unconstrained gradient descent has a simple closed form:

\newcommand{\lemgdrewardtxt}{ Consider unconstrained gradient descent
  in one dimension, with learning rate $\eta$.  On round $t$, this
  algorithm plays the point $x_t = \eta g_{1:t-1}$.  Letting $G =
  |g_{1:t}|$ and $H = \sum_{t=1}^T g_t^2$, the cumulative reward of
  the algorithm is exactly
  \[
  \Reward = \frac \eta 2 \p { G^2 - H }.
  \]
}

\begin{lemma}\label{lem:gd-reward}
\lemgdrewardtxt
\end {lemma}
We give a simple direct proof in Appendix A.  Perhaps surprisingly,
this result implies that the reward is totally independent of the
order of the linear functions selected by the adversary.
Examining the expression in Lemma \ref{lem:gd-reward}, we see that the
optimal choice of learning rate $\eta$ depends fundamentally on two
quantities: the absolute value of the sum of gradients ($G$), and the
sum of the squared gradients ($H$).  If $G^2 > H$, we would like to
use as large a learning rate as possible in order to maximize reward.
In contrast, if $G^2 < H$, the algorithm will obtain negative reward,
and the best it can do is to cut its losses by setting $\eta$ as small
as possible.

One of the motivations for this work is the observation that the
state-of-the-art online gradient descent algorithms adjust their
learning rates based \emph {only} on the observed value of $H$ (or
its upper bound $T$); for
example~\citep{duchi10adaptive,mcmahan10boundopt}.
We would like to increase reward by also accounting for $G$.
But unlike $H$, which is monotonically increasing with time, $G$ can
both increase and decrease.
{\bf This makes simple guess-and-doubling tricks fail}
when applied to $G$, and necessitates a more careful
approach.

\subsection{Analysis in One Dimension}

In this section we analyze algorithm \ALGa (Algorithm~\ref{alg:a}),
which consists of a series of epochs.  We suppose for the moment that
an upper bound $\bar H$ on $H = \sum_{t=1}^T g_t^2$ is known in
advance.  In the first epoch, we run gradient descent with a small
initial learning rate $\eta = \eta_1$.  Whenever the total reward
accumulated in the current epoch reaches $\eta \bar H$, we double
$\eta$ and start a new epoch (returning to the origin and forgetting
all previous gradients except the most recent one).

\newcommand{\assign}{\leftarrow}
\begin{figure*}[t]
\begin{minipage}[t]{2.5in}
\begin{algorithm}[H]
\caption{\ALGa} \label{alg:a}
\begin{algorithmic}
\STATE {\bf Parameters:} initial learning rate $\eta_1$,
 upper bound $\bar H \ge \sum_{t=1}^T g_t^2$.
\STATE Initialize $x_1 \assign 0$, $i \assign 1$, and $Q_1 \assign 0$.
\FOR{ $t = 1, 2, \ldots, T $}
  \STATE Play $x_t$, and receive reward $x_t g_t$.
  \STATE $Q_i \assign Q_i + x_t g_t$. 
  \IF {$Q_i < \eta_i \bar H$}
  \STATE $x_{t+1} \assign x_t + \eta_i g_t$.
  \ELSE
  \STATE $i \assign i + 1$.
  \STATE $\eta_i \assign 2 \eta_{i-1}$; $Q_i \assign 0$.
  \STATE $x_{t+1} \assign 0 + \eta_i g_t$.
  \ENDIF
\ENDFOR
\end{algorithmic}
\end{algorithm}
\end{minipage}
\hfill
\begin{minipage}[t]{2.5in}
\begin{algorithm}[H]
\caption{\ALGb} \label{alg:b}
\begin{algorithmic}
\STATE {\bf Parameters:} maximum origin-regret $\eps_i$ for $1 \le i \le n$.
\FOR{ $i = 1, 2, \ldots, n$}
  \STATE Let $A_i$ be a copy of algorithm \ALGaguess (see Theorem~\ref{thm:double}), with parameter $\eps_i$.
\ENDFOR
\FOR{ $t = 1, 2, \ldots, T $}
  \STATE Play $x_t$, with $x_{t,i}$ selected by $A_i$.
  \STATE Receive gradient vector $g_t = -\grad f_t(x_t)$.
  \FOR{ $i = 1, 2, \ldots, n$}
    \STATE Feed back $g_{t,i}$ to $A_i$.
  \ENDFOR
\ENDFOR
\end{algorithmic}
\end{algorithm}
\end{minipage}
\hfill

\end{figure*}

\begin {lemma}\label{lem:knownHReward}
  Applied to a sequence of gradients $g_1, g_2, \ldots, g_T$, all in
  $[-1,1]$, where $H = \sum_{t=1}^T g_t^2 \le \bar H$, \ALGa obtains
  reward satisfying
  \begin{equation}\label{eq:alg1reward}
    \Reward = \sum_{t=1}^T x_t g_t
    \ge \frac 1 4 \eta_1 \bar H \bexp{a \frac{|g_{1:T}|}{\sqrt{\bar H}}}
    - \eta_1 \bar H,
  \end{equation}
  for $a = \log(2)/\sqrt{3}$.
\end {lemma}

\begin{proof}  %
  Suppose round $T$ occurs during the $k$'th epoch.  Because epoch $i$
  can only come to an end if $Q_i \ge \eta_i \bar H$, where $\eta_i =
  2^{i-1} \eta_1$, we have
  \begin{equation}\label{eq:firstbound}
  \Reward = \sum_{i=1}^k Q_i
  \ge
  \p { \sum_{i=1}^{k-1} 2^{i-1} \eta_1 \bar H } + Q_k
  =
  \p { 2^{k-1}-1 } \eta_1 \bar H + Q_k \mbox { .}
  \end{equation}

  We now lower bound $Q_k$.  For $i=1, \dots, k$ let $t_i$ denote the
  round on which $Q_i$ is initialized to 0, with $t_1 \equiv 1$, and
  define $t_{k+1} \equiv T$.  By construction, $Q_i$ is the total
  reward of a gradient descent algorithm that is active on rounds
  $t_i$ through $t_{i+1}$ inclusive, and that uses learning rate
  $\eta_i$ (note that on round $t_i$, this algorithm gets 0 reward and
  we initialize $Q_i$ to 0 on that round).  Thus, by Lemma \ref
  {lem:gd-reward}, we have that for any $i$,
  \[
    Q_i =
    \frac {\eta_i} {2} \p { (g_{t_i:t_{i+1}})^2 - \sum_{s=t_i}^{t_{i+1}} g_s^2 }
    \ge 
    -\frac {\eta_i} {2} \bar H \mbox { .}
  \]
  Applying this bound to epoch $k$, we have
  $
  Q_k \ge - \frac 1 2 \eta_k \bar H = -2^{k-2} \eta_1 \bar H
  $.
  Substituting into \eqref{eq:firstbound} gives
  \begin{equation}\label{eq:min-reward}
  \Reward \ge \eta_1 \bar H (2^{k-1} - 1 - 2^{k-2}) = \eta_1 \bar H (2^{k-2} - 1)
\mbox { .}
  \end{equation}
  We now show that $k \ge \frac {|g_{1:T}|} {\sqrt {3\bar H} }$.  At
  the end of round $t_{i+1}-1$, we must have had $Q_i < \eta_i \bar H$
  (otherwise epoch $i+1$ would have begun earlier).  Thus, again using
  Lemma \ref{lem:gd-reward},
  \[
  \frac {\eta_i} {2} \p { (g_{t_i:t_{i+1} - 1})^2 - \bar H } \le \eta_i \bar H
  \]
  so $|g_{t_i:t_{i+1} - 1}| \le \sqrt {3 \bar H}$.  Thus,
  \[
  |g_{1:T}| \le \sum_{i=1}^k |g_{t_i:t_{i+1}-1}| \le k \sqrt {3\bar H} \mbox { .}
  \]
  Rearranging gives $k \ge \frac {|g_{1:T}|} {\sqrt {3\bar H} }$, and
  combining with \eqr{min-reward} proves the lemma.
\end {proof}

We can now apply Theorem~\ref{thm:rb} to the reward (given by
\eqr{alg1reward}) of \ALGa to show
\begin{equation}\label{eq:orb}  
\Regret(\xs) \leq b R \sqrt{\bar H} 
\p { \log\left(\frac{4 R b \sqrt{\bar H}}{\eta_1}\right) - 1 }
+ \eta_1 \bar H
\end{equation}
for any $\xs \in [-R,R]$, where $b = a^{-1} = \sqrt{3}/\log(2) <
2.5$.  When the feasible set is also fixed in advance, online gradient
descent with a fixed learning obtains a regret bound of
$\BO(R\sqrt{T})$.  Suppose we use the estimate $\bar H = T$.  By
choosing $\eta_1 = \frac{1}{T}$, we guarantee constant regret against
the origin, $\xs = 0$ (equivalently, constant total loss).  Further,
for \emph{any} feasible set of radius $R$, we still have worst-case
regret of at most $\BO(R \sqrt{T} \log((1+R) T))$, which is only
modestly worse than that of gradient descent with the optimal $R$
known in advance.

The need for an upper bound $\bar H$ can be removed using a standard
guess-and-doubling approach, at the cost of a constant factor increase
in regret (see appendix for proof).
\newcommand{\thmdoubletxt}{ Consider
  algorithm \ALGaguess, which behaves as follows.
On each era $i$, the algorithm
  runs \ALGa with an upper bound of $\bar H_i =
  2^{i-1}$, and initial learning rate $\eta_1^i = \eps 2^{-2i}$.  An
  era ends when $\bar H_i$ is no longer an upper bound on the sum of
  squared gradients seen during that era.  Letting $c = \frac {\sqrt
    2} {\sqrt 2 - 1}$, this algorithm has regret at most
  \[
  \Regret \le 
  c R \sqrt {H+1} \p { \log\p{\frac R \eps (2H+2)^{5/2}} - 1 } + \eps.
  \]
}
\begin{theorem}\label{thm:double}
\thmdoubletxt
\end{theorem}

\subsection{Extension to $n$ dimensions}

To extend our results to general online convex optimization, it is
sufficient to run a separate copy of \ALGaguess for each coordinate,
as is done in \ALGb (Algorithm~\ref{alg:b}).  The key to the analysis
of this algorithm is that overall regret is simply the sum of regret
on $n$ one-dimensional subproblems which can be analyzed
independently.

\begin {theorem}\label{thm:ndim}
  Given a sequence of convex loss functions $f_1, f_2, \ldots, f_T$
  from $\R^n$ to $\R$, \ALGb with $\eps_i = \frac \eps n$ has regret
  bounded by
\begin {align*}
\Regret(\xs)
& \le 
\eps + c \sum_{i=1}^n |\xs_i| \sqrt {H_i+1} \p { \log\p{\frac n \eps |\xs_i| (2H_i+2)^{5/2}} - 1 } \\
& \le \eps + c \norm{\xs}_2 \sqrt {H+n} \p { \log\p{\frac n \eps \norm{\xs}_2^2
(2H+2)^{5/2}} - 1 }
\end {align*}
for $c = \frac {\sqrt 2} {\sqrt 2 - 1}$, where $H_i = \sum_{t=1}^T g_{t,i}^2$
and $H = \sum_{t=1}^T \norm{g_t}_2^2$.
\end {theorem}

\begin {proof}
Fix a comparator $\xs$.  For any coordinate $i$, define
\[
\Regret_i = \sum_{t=1}^T \xs_i g_{t,i} - \sum_{t=1}^T x_{t,i} g_{t,i} \mbox { .}
\]
Observe that
\[
\sum_{i=1}^n \Regret_i = \sum_{t=1}^T \xs \cdot g_t - \sum_{t=1}^T x_t \cdot g_t
 = \Regret(\xs) \mbox { .}
\]
Furthermore, $\Regret_i$ is simply the regret of \ALGaguess
on the gradient sequence $g_{1,i}, g_{2,i},
\ldots, g_{T,i}$.  Applying the bound of Theorem~\ref{thm:double} to
each $\Regret_i$ term completes the proof of the first inequality.
For the second inequality, let $\vec H$ be a vector whose $i^{th}$
component is $\sqrt{H_i+1}$, and let $\vec x \in \R^n$ where $\vec x_i =
\abs{\xs_i}$.  Using the Cauchy-Schwarz inequality, we have
\[
\sum_{i=1}^n |\xs_i| \sqrt {H_i + 1}
= \vec x \cdot \vec H
\le \norm{\xs}_2 \, \norm{\vec H}_2 
= \norm{\xs}_2 \, \sqrt{H + n}
\mbox { .}
\]
This, together with the fact that
$\log(|\xs_i| (2H_i+2)^{5/2}) \le \log(\norm{\xs}_2^2 (2H+2)^{5/2})$,
suffices to prove second inequality.
\end {proof}

In some applications, $n$ is not known in advance.  In this case, we
can set $\eps_i = \frac {\eps} {i^2}$ for the $i$th
coordinate we encounter, and get the same bound up to constant factors.

\section{An Epoch-Free Algorithm}
\label{sec:epochfree}
\newcommand{\sGt}{g_{1:t}} In this section we analyze \ALGs, a simple
algorithm that achieves bounds comparable to those of
Theorem~\ref{thm:double}, without guessing-and-doubling.  We consider
only the 1-d problem, as the technique of Theorem~\ref{thm:ndim} can
be applied to extend to $n$ dimensions.  Given a parameter $\eta > 0$, we achieve
\begin{equation}\label{eq:dreg}
 \Regret \leq R \sqrt{T} \left(\log\p{\frac {RT^{3/2}} {\eta}} - 1\right) 
     + 1.76 \eta,
\end{equation}
for all $T$ and $R$, which is better
(by constant factors) than Theorem~\ref{thm:double} when $g_t \in
\{-1, 1\}$ (which implies $T = H$).  The bound can be worse on a
problems where $ H < T$.

The idea of the algorithm is to maintain the invariant that our
cumulative reward, as a function of $g_{1:t}$ and $t$,
satisfies $\Reward \geq N(\sGt, t)$, for some fixed function $N$.
Because reward changes by $g_t x_t$ on round $t$, it suffices
to guarantee that for any $g \in [-1,1]$,
\begin{equation}\label{eq:maxmin}
N(\sGt, t) + g x_{t+1} \ge N(\sGt + g, t+1)
\end{equation}
where $x_{t+1}$ is the point the algorithm plays on round $t+1$, and
we assume $N(0, 1) = 0$.

This inequality is approximately satisfied
(for small $g$) if we choose
\[
x\ti = \frac {\partial N(\sGt + g, t)} {\partial g}
\approx \frac { N(\sGt + g, t) - N(\sGt, t) } {g}
\approx \frac { N(\sGt + g, t+1) - N(\sGt, t) } {g}
\mbox { .}
\]
This suggests that if we want to maintain reward at least $N(\sGt, t)
= \frac{1}{t} (\exp(\abs{\sGt}/\sqrt{t})-1)$ , we should set $x_{t+1}
\approx \sign(g_{1:t}) t^{-3/2} \exp\p{\frac{|\sGt|}{\sqrt t}}$.  The
following theorem (proved in the appendix) provides an inductive
analysis of an algorithm of this form.

\newcommand{\thmepochfreetxt}[2]{
  Fix a sequence of reward functions $f_t(x) = g_t x$ with $g_t \in
  [-1, 1]$, and let $G_t = \abs{g_{1:t}}$.  We consider \ALGs, which
  plays $0$ on round $1$ and whenever $G_t = 0$; otherwise, it plays
  \begin{equation} #1
    x_{t+1} = \eta \sign(g_{1:t})B(G_t, t + 5)
  \end{equation}
  with $\eta > 0$ a learning-rate parameter and 
  \begin{equation} #2
  B(G, t) = \frac{1}{t^{3/2}}\bexp{\frac{G}{\sqrt{t}}}.
  \end{equation}
  Then, at the end of each round $t$, this algorithm has
  \[ 
  \Reward(t) \geq  \eta \frac{1}{t+5}\bexp{\frac{G_t}{\sqrt{t+5}}} - 1.76\eta.
  \] 
}
\begin{theorem} \label{thm:epoch-free}
\thmepochfreetxt{\label{eq:dbt}}{\label{eq:B}}
\end{theorem}

Two main technical challenges arise in the proof: first, we prove a
result like \eqr{maxmin} for $N(\sGt, t) =
(1/t)\exp\big(\abs{g_{1:t}}/\sqrt{t}\big)$.  However, this Lemma only
holds for $t \geq 6$ and when the sign of $g_{1:t}$ doesn't change.
We account for this by showing that a small modification to $N$
(costing only a constant over all rounds) suffices.

By running this algorithm independently for each coordinate using an
appropriate choice of $\eta$, one can obtain a guarantee similar to
that of Theorem~\ref{thm:ndim}.

\ignore {
\begin{algorithm}[t]
\caption {A simple epoch-free algorithm for unconstrained online
  convex optimization.} \label {alg:direct}
\begin{algorithmic}
\STATE Parameter: learning rate $\eta$
\STATE Initialize $g_0 \assign 0$.
\FOR{ $t = 1, 2, \ldots, T $}
  \STATE Play $x_t$, where $x_{t,i} = \eta \frac{\sign(g_{1:t-1})}{(t + 4)^{3/2}}
     \bexp{\frac{|g_{1:t-1}|}{\sqrt{t + 4}}}$
  \STATE Suffer loss $f_t(x_t)$, receive gradient vector $g_t = -\grad f_t(x_t)$.
  \STATE $g_{1:t} \assign g_{1:t-1} + g_t$.
\ENDFOR
\end{algorithmic}
\end {algorithm}
}

\section{Lower Bounds} \label{sec:lower}

As with our previous results, it is sufficient to show a lower bound
in one dimension, as it can then be replicated independently in each
coordinate to obtain an $n$ dimensional bound.  Note that our lower
bound contains the factor $\log(|\xs|\sqrt T)$, which can be negative
when $\xs$ is small relative to $T$, hence it is important to hold
$\xs$ fixed and consider the behavior as $T \rightarrow \infty$.  Here
we give only a proof sketch; see Appendix A for the full proof.

\newcommand{\lemlowertext}{ 
  Consider the problem of unconstrained online linear optimization in
  one dimension, and an online algorithm that guarantees
  origin-regret at most $\eps$.  Then, for any fixed comparator $\xs$,
  and any integer $T_0$, there exists a gradient sequence $\{g_t\} \in
  [-1,1]^T$ of length $T \ge T_0$ for which the algorithm's regret
  satisfies
\[
\Regret(\xs) \ge 0.336 |\xs| \sqrt {T \log\p{\frac {|\xs| \sqrt T} {\eps}}} \mbox { .}
\]
}
\begin{theorem} \label {thm:lower1d}
\lemlowertext
\end{theorem}
\begin {proof} (Sketch)
Assume without loss of generality that $\xs > 0$.  Let $Q$ be the
algorithm's reward when each $g_t$ is drawn independently uniformly
from $\{-1,1\}$.  We have $\E[Q] = 0$, and because the algorithm
guarantees origin-regret at most $\eps$, we have $Q \ge -\eps$ with
probability 1.  Letting $G = g_{1:T}$, it follows that for any
threshold $Z=Z(T)$,
\begin {align*}
0
& = \E[Q] \\
& = \E[Q | G < Z] \cdot \Pr[G < Z]
+ \E[Q | G \ge Z] \cdot \Pr[G \ge Z] \\
& \ge -\eps \Pr[G < Z] + \E[Q | G \ge Z] \cdot \Pr[G \ge Z] \\
& > -\eps + \E[Q | G \ge Z] \cdot \Pr[G \ge Z] \mbox { .}
\end {align*}
Equivalently,
\[
\E[Q | G \ge Z] < \frac {\eps} {\Pr[G \ge Z]} \mbox { .}
\]
We choose $Z(T) = \sqrt{kT}$, where $k = \floor{\log(\frac {R \sqrt T}
  {\eps})/\log(p^{-1})}$.  Here $R = \abs{\xs}$ and $p > 0$ is a
constant chosen using binomial distribution lower bounds so that
$\Pr[G \ge Z] \ge p^k$.  This implies
\[
\E[Q | G \ge Z] < \eps p^{-k} = \eps \exp\p{k \log p^{-1}}
\le R \sqrt T \mbox { .}
\]
This implies there exists a sequence with $G \ge Z$ and $Q < R\sqrt
T$.  On this sequence, regret is at least $G \xs -
Q \ge R \sqrt{kT} - R \sqrt T = \Omega(R \sqrt{kT})$.
\end {proof}

\begin {theorem} \label {thm:lower}
Consider the problem of unconstrained online linear optimization in
$\R^n$, and consider an online algorithm that guarantees origin-regret
at most $\eps$.  For any radius $R$, and any $T_0$, there exists a
gradient sequence gradient sequence $\{g_t\} \in ([-1,1]^n)^T$ of
length $T \ge T_0$, and a comparator $\xs$ with $\norm{\xs}_1 = R$,
for which the algorithm's regret satisfies
  \[
  \Regret(\xs) \ge
  0.336 \sum_{i=1}^n |\xs_i| \sqrt {T \log\p{\frac {|\xs_i| \sqrt T} {\eps}}} 
\mbox{ .}
  \]
\end {theorem}
\begin {proof}
For each coordinate $i$, Theorem~\ref{thm:lower1d} implies that there exists
a $T \ge T_0$ and a sequence of gradients $g_{t,i}$ such that
\[
\sum_{t=1}^T \xs_i g_{t,i} - \sum_{t=1}^T x_{t,i} g_{t,i} 
\ge 0.336 |\xs_i| \sqrt {T \log\p{\frac {|\xs_i| \sqrt T} {\eps}}} \mbox { .}
\]
(The proof of Theorem~\ref{thm:lower1d} makes it clear that we can use
the same $T$ for all $i$.)
Summing this inequality across all $n$ coordinates then gives the
regret bound stated in the theorem.
\end {proof}

The following theorem presents a stronger negative result for
Follow-the-Regularized-Leader algorithms with a fixed regularizer: for
any such algorithm that guarantees origin-regret at most $\eps_T$
after $T$ rounds, worst-case regret with respect to any point outside
$[-\eps_T, \eps_T]$ grows linearly with $T$.
\newcommand{\thmftrlbadtxt}{ Consider a Follow-The-Regularized-Leader
  algorithm that sets
\[
x_t = \argmin_{x} \p { g_{1:t-1} x + \psi_T(x) }
\]
where $\psi_T$ is a convex, non-negative function with
$\psi_T(0) = 0$.
Let $\eps_T$ be the maximum origin-regret incurred by the algorithm
on a sequence of $T$ gradients.  Then, for any $\xs$ with $|\xs| > \eps_T$,
there exists a sequence of $T$ gradients such that the algorithm's regret
with respect to $\xs$ is at least $\frac{T-1}{2} (|\xs| - \eps_T)$.
}

\begin {theorem} \label{thm:ftrl-bad}
\thmftrlbadtxt
\end {theorem}
In fact, it is clear from the proof that the above result holds for
any algorithm that selects $x_{t+1}$ purely as a function of $g_{1:t}$
(in particular, with no dependence on $t$).

\section{Future Work}
This work leaves open many interesting questions.  It should be
possible to apply our techniques to problems that do have constrained
feasible sets; for example, it is natural to consider the
unconstrained experts problem on the positive orthant.  While we
believe this extension is straightforward, handling arbitrary
non-axis-aligned constraints will be more difficult.  Another
possibility is to develop an algorithm with bounds in terms of $H$
rather than $T$ that doesn't use a guess and double approach.

\newpage

\bibliographystyle{plainnat}
\bibliography{inc}

\supplementary {
\appendix
\newpage
\section{Proofs}
This appendix gives the proofs omitted in the body of the paper, with
the corresponding lemmas and theorems restated for convenience.

\begin{reptheorem}{thm:rb}
\thmrbtxt{\tag{\ref{eq:rewardbound}}}{\tag{\ref{eq:regbound}}}
\end{reptheorem}

\begin{proof}
  Let $G_T = \abs{g_{1:T}}$.  By definition, given the reward guarantee
  of \eqr{rewardbound} we have
  \begin{equation}\label{eq:fromdefr}
    \Regret \leq RG_T - \kappa \bexp{\gamma G_T} + \epsilon.
  \end{equation}
  If $R=0$, then \eqr{regbound} follows immediately.  Otherwise, note
  this is a concave function in $G_T$, and setting the first
  derivative equal to zero shows
  \[G^* = \frac{1}{\gamma} 
  \log\left(\frac{R}{\gamma \kappa}\right).
  \]
  maximizes regret (for large enough $R$ we could have $G^* > T$, and
  so this $G^*$ is not actually achievable by the adversary, but this
  is fine for lower bounding regret).  Plugging $G^*$ into
  \eqr{fromdefr} and simplifying yields the bound of \eqr{regbound}.
  For the second claim, suppose \eqr{regbound} holds.  Then, again by
  definition, we must have
  \begin{equation}\label{eq:r2}
    \Reward \geq R G - \frac{R}{\gamma}
    \log\left(\frac{R}{\gamma \kappa}\right) 
    + \frac{R}{\gamma} - \epsilon.
  \end{equation}
  This bound is a concave function of $R$, and since it holds for any $R
  \geq 0$ by assumption, we can choose the $R$ that maximizes the bound,
  namely $R^* = \gamma \kappa \exp( \gamma G)$.  Note
  \[
  \frac{R^*}{\gamma}   
  \log\left(\frac{R^*}{\gamma \kappa}\right)
  = \frac{R^*}{\gamma}  
  \log\p{\exp\p{\gamma G}} = R^* G,
  \]
  and so plugging $R^*$ into \eqr{r2} yields  
  \[
  \Reward \geq \frac{1}{\gamma}R^* - \epsilon 
  = \kappa \bexp{\gamma G} - \epsilon.
  \]
\end{proof}

\begin{replemma}{lem:gd-reward}
\lemgdrewardtxt
\end{replemma}
\begin{proof}
The algorithm's cumulative reward after $T$ rounds is
\begin {equation} \label {eq:factorization}
\sum_{t=1}^T x_t g_t = \sum_{t=1}^T g_t \eta g_{1:t-1}
= \frac \eta 2 \p { (g_{1:T})^2 - \sum_{t=1}^T g_t^2 } \mbox { .}
\end {equation}
To verify the second equality, note that $(g_{1:T})^2 - (g_{1:T-1})^2 =
g_T^2 + 2 g_T (g_{1:T-1})$, so on round $T$ the right hand side
increases by $\eta g_T (g_{1:T-1})$, as does the left hand side.  The
equality then follows by induction on $T$.
\end {proof}

It is worth noting that the standard $R\sqrt{T}$ bound can be derived
from the above result fairly easily.  We have
\begin{align*}
 \Regret 
   &\leq RG - \frac{\eta}{2}(G^2 - H) \\
   &\leq \frac{\eta}{2} H + \max_G \left(RG - \frac{\eta G^2}{2} \right)\\
   &\leq \frac{\eta}{2} H + \frac{R^2}{2\eta},
\end{align*}
where the max is achieved by taking $G = R/\eta$.  Taking $\eta =
R/\sqrt{T}$ then gives the standard bound.  However, this bound
significantly underestimates the performance of constant-learning-rate
gradient descent when $G$ is large.
This is in contrast to our regret
bounds, which are always tight with respect to their matching reward
bounds.

\begin{reptheorem}{thm:double}
\thmdoubletxt
\end{reptheorem}
\begin{proof}
  Suppose round $T$ occurs in era $k$, and let $t_i$ be the round on
  which era $i$ starts, with $t_{k+1} \equiv T+1$.  Define $H_i =
  \sum_{s=t_i}^{t_{i+1}-1} g_s^2$.  To prove the theorem we will need
  several inequalities.  First, note that $H = \sum_{i=1}^k H_i \ge
  \sum_{i=1}^{k-1} \bar H_i = 2^{k-1} - 1$, or $2^{k-1} \le H+1$.
  Thus,
  \[
  \sum_{i=1}^k \sqrt {\bar H_i} = \sum_{i=0}^{k-1} \sqrt {2^i}
  = \frac{\sqrt{2}^k - 1}{\sqrt{2}-1}
  \le \frac{\sqrt{2^k}}{\sqrt{2}-1}
  \le \frac{\sqrt{2(H+1)}}{\sqrt{2}-1}
  = c \sqrt {H+1} \mbox { .}
  \]
  Next, note that for any $i$ we have
  \[
  \frac { \sqrt {\bar H_i} } {\eta_1^i}
  = \frac 1 \eps 2^{\frac {i-1} {2} + 2i}
  \le \frac 1 \eps 2^{2.5k} \le \frac 1 \eps (2(H+1))^{(5/2)}.
  \]
  Note that the bound of Lemma~\ref{lem:knownHReward} applies for all
  $T$ where $H \leq \bar{H}$, and thus so does \eqr{orb}.  Thus, we
  can apply this bound to the regret in era $k$ on rounds $t_k$
  through $T$, as well as on the regret in each earlier era.  Then,
  total regret with respect to the best point in $[-R,R]$ is at most
  the sum of the regret in each era, so
  \begin {align*}
    \Regret
    & \le \sum_{i=1}^k R \sqrt {\bar H_i} 
    \p { \log\p{\frac {R \sqrt {\bar H_i}} {\eta_1^i}} - 1 } + \eta_1^i H_i \\
    & \le \sum_{i=1}^k R \sqrt {\bar H_i} 
      \p { \log\p{\frac R \eps (2H+2)^{5/2} } - 1 } + \eta_1^i H_i \\
    & \le c R \sqrt {H+1} \p { \log\p{\frac R \eps (2H+2)^{5/2}} - 1 }  
        + \sum_{i=1}^k \eta_1^i H_i
  \end {align*}
  Finally, because $H_i \le \bar H_i + 1 \le 2 \bar H_i = 2^i$, we have
  $\sum_{i=1}^k \eta_1^i H_i \le \sum_{i=1}^k \eps 2^{-i} \le \eps$, which completes
  the proof.
\end {proof}

\begin{reptheorem}{thm:epoch-free}
\thmepochfreetxt{\tag{\ref{eq:dbt}}}{\tag{\ref{eq:B}}}
\end{reptheorem}

\begin{proof}
  We present a proof for the case where $\eta = 1$; since $\eta$
  simply scales all of the $x_t$ played by the algorithm (and hence,
  reward), the result for general $\eta$ follows immediately.  We use
  the minimum reward function
  \begin{equation}\label{eq:N}
  N(G, t) = \frac{1}{t}\bexp{\frac{G}{\sqrt{t}}}.
  \end{equation}
  The proof will be by induction on $t$, with the induction hypothesis
  that the cumulative reward of the algorithm at the end of round $t$
  satisfies
  \begin{equation}\label{eq:ih}
  \Reward(t) \geq  N(G_t, t+5) -\eps_{1:t},
  \end{equation}
  where $\eps_1 = N(1,6)$  and for $t > 1$, $\eps_{t+1} = \heps(t+5)$ 
  with
  \[
  \heps(\tau) = \frac{1}{\tau+1}\bexp{\frac{1}{\sqrt{\tau+1}}} 
  -\frac{1}{\tau} +\frac{1}{\tau^{3/2}}.
  \]
  We will then show that the sum of $\eps_t$'s is always bounded by
  a constant.

  For the base case, $t=1$, we play $x=0$ so end the round with zero
  reward, while the RHS of \eqr{ih} is $N(\abs{g_1}, 6) - N(1, 6) \le
  0$.
  
  Now, suppose the induction hypothesis holds at the end of some round
  $t \ge 1$.  Without loss of generality, suppose $g_{1:t} \geq 0$ so
  $G_t = g_{1:t}$.  We consider two cases.  First, suppose $G_t > 0$
  and $G_t + g\ti > 0$ (so $g\ti > -G_t$).  In this case, $g_{1:t}$
  does not change sign when we add $g\ti$; thus, an invariant like
  that of \eqr{maxmin} is sufficient; we prove such a result in
  Lemma~\ref{lem:rinv} (given below). More precisely, we
  play $x_{t+1}$ according to \eqr{dbt}, and
  \begin{align*}
    \Reward(t+1) 
      &\geq N(G_t, t+5) -\eps_{1:t} + g\ti x\ti  &&  \text{IH and update rule}\\
      &\geq N(G_t + g\ti, t+5 + 1)  -\eps_{1:t} 
        && \text{Lemma~\ref{lem:rinv} with $\tau=t+5$.} \\
      &\geq N(G_{t+1}, t+5 + 1)  - \eps_{1:t+1},  &&\text{since $\eps\ti > 0$}.
  \end{align*}

  For the remaining case, we have $G_t + g\ti \leq 0$, implying $g\ti
  \leq -G_t \leq 0$.  In this case, we suffer some loss and arrive at
  $G_{t+1} = \abs{G_t + g\ti} = -g\ti - G_t$.  Lemma~\ref{lem:od}
  (below) provides the key bound on the additional loss when the sign
  of $g_{1:t}$ changes.  If $G_t > 0$, we have
  \begin{align*} 
     \Reward(t+1) 
        &\geq N(G_t, t+5) - \eps_{1:t} + g\ti x\ti &&  \text{IH and update rule}\\
        &\geq N(-g\ti - G_t, t+5+1) - \eps_{1:t+1}  
          && \text{Lemma~\ref{lem:od} with $\tau = t+5$} \\
        &= N(G\ti, t+5+1) - \eps_{1:t+1}.
  \end{align*}
  If $G_t = 0$, we can take $g\ti$ non-positive without loss of
  generality, and playing $x\ti = 0$ is no worse than playing $B(0,
  t+5)$, and so we conclude \eqr{ih} holds for all $t$.
  Finally, 
  \begin{align*}
    \sum_{t=2}^\infty \eps_t
    &\leq \int_{\tau=6}^\infty \heps(\tau)
    = \sqrt{\frac{2}{3}} -2 \gamma + 2\Ei\left(\frac{1}{\sqrt{7}}\right) 
       + \log(6) \leq 1.50.
  \end{align*}
  where $\gamma$ is the Euler gamma constant and $\Ei$ is the
  exponential integral.  The upper bound can be found easily using
  numerical methods.  Adding $\eps_1 = \exp(1/\sqrt{6})/6 \leq 0.26$
  gives $\eps_{1:T} \leq 1.76$ for any $T$.
\end{proof}

\begin{lemma}\label{lem:rinv}
  Let $G > 0$ and $\tau \geq 6$.  Then, for any $g \in [-1,
  1]$ such that $G + g \geq 0$,
  \[ 
  N(G, \tau) + g B(G,\tau) - N(G + g, \tau+1) \geq 0 
  \] 
  where $N$ is defined by \eqr{N} and $B$ is defined by \eqr{B}.
\end{lemma}

\begin{proof}
We need to show
\[
\frac{1}{\tau} \exp \left(\frac{G}{\sqrt{\tau}} \right)
+ \frac{g}{\tau^{3/2}}\exp\left(\frac{G}{\sqrt{\tau}}\right)
 - \frac{1}{\tau + 1} \exp\left(\frac{G + g}{\sqrt{\tau + 1}} \right)
\geq 0.
\]
or equivalently, multiplying by $\tau^{3/2} (1 + \tau)/\exp(G/\sqrt{\tau}) \geq 0$,
\[
\Delta = 
 \sqrt{\tau}(1+ \tau) + g (1 + \tau) 
  - \tau^{3/2}\exp\left(
   \frac{G + g}{\sqrt{\tau + 1}} - \frac{G}{\sqrt{\tau}} 
   \right) \geq 0.
\]
Since $\tau + 1 \geq \tau$, the $\exp$ term is maximized when $G = 0$, so
\begin{align}
  \Delta &\geq
  (g + \sqrt{\tau})(1+ \tau)- \tau^{3/2}\exp\left(\frac{g}{\sqrt{\tau + 1}}\right). 
  \label{eq:Delta} \\
  \intertext{Now, we consider the cases where $g \geq 0$ and $g < 0$ separately. 
    First, suppose $g > 0$, so $g / \sqrt{\tau + 1} \in [0, 1]$, 
    and we can use the inequality $\exp(x) \leq 1 + x + x^2$ for $x \in [0, 1]$,
    which gives} 
\Delta 
  &\geq g + g \tau + \sqrt{\tau} + \tau^{3/2}
  - \tau^{3/2} \left(1 + \frac{g}{\sqrt{\tau + 1}} + \frac{g^2}{\tau + 1}\right) \notag \\
  &\geq g + g \tau + \sqrt{\tau} + \tau^{3/2}
  - \tau^{3/2} \left(1 + \frac{g}{\sqrt{\tau}} + \frac{1}{\tau}\right) \notag \\
  &= g + g \tau + \sqrt{\tau} + \tau^{3/2} - \tau^{3/2} - g \tau - \sqrt{\tau}  \notag \\
  &= g > 0. \notag
\end{align}
Now, we consider the case where $g < 0$.  In order to show $\Delta
\geq 0$ in this case, we need a tight upper bound on $\exp(y)$ for $y
\in [-1,0]$.  To derive one, we note that for $x \ge 0$,
$\exp(x) \geq 1 + x +\h x^2$ from
the series representation of $e^x$, and so $\exp(-x) \leq (1 + x +
\h x^2)^{-1}$.  Thus, for $y \in [-1, 0]$ we have $\exp(y) \leq (1 - y +
\h y^2)^{-1} = Q(y)$.  Then, starting from \eqr{Delta},
\[
  \Delta \geq
  (g + \sqrt{\tau})(1+ \tau)- \tau^{3/2}Q\left(\frac{g}{\sqrt{\tau + 1}}\right).
\]
Let $\Delta_2 = \Delta Q\left(\frac{g}{\sqrt{\tau + 1}}\right)^{-1}$.
Because $\Delta_2$ and $\Delta$ have the same sign, it suffices to show
$\Delta_2 \ge 0$.  We have
\begin{align*}
\Delta_2 &= 
  \left(1 - \frac{g}{\sqrt{\tau+1}} + \frac{g^2}{2(t+1)}\right)
  (g + \sqrt{\tau})(1+ \tau)- \tau^{3/2}  \\
  &= \big(1 + \tau - g\sqrt{\tau+1} + \h g^2\big) (g + \sqrt{\tau}) - \tau^{3/2} .
\end{align*}
First, note
\[ \frac{d}{dg} \Delta_2 = 
 1+\frac{3 g^2}{2}+g \sqrt{\tau}+ \tau-2 g \sqrt{1+ \tau}
   -\sqrt{\tau} \sqrt{1+ \tau}.
\]
Since $g \leq 0$, we have $-2g\sqrt{\tau+1} + g\sqrt{\tau} \geq 0$, and
$(t+1) - \sqrt{\tau}\sqrt{\tau+1} \geq 0$, and so we conclude that
$\Delta_2$ is increasing in $g$, and so taking $g = -1$ we have
\[\Delta_2 \geq 
 \big(\frac{3}{2} + \tau + \sqrt{\tau+1} \big) (-1 + \sqrt{\tau}) - \tau^{3/2}
\]
Taking the derivative with respect to $\tau$ reveals this expression is
increasing in $\tau$, and taking $\tau = 6$ produces a positive value,
proving this case.
\end{proof}

\begin{lemma}\label{lem:od}
  For any $g \in [-1,0]$ and $G \geq 0$ such that $G + g \leq 0$, and
  any $\tau \geq 1$,
  \[ 
  N(G, \tau) + g B(G, \tau) \geq N(-g - G, \tau+1) - \heps(\tau)
  \] 
  where $N$ is defined by \eqr{N} and $B$ is defined by \eqr{B}, and
  \[
  \heps(\tau) \equiv  \frac{1}{\tau+1}\bexp{\frac{1}{\sqrt{\tau+1}}} 
           -\frac{1}{\tau} +\frac{1}{\tau^{3/2}}.
  \]
\end{lemma}

\begin{proof}
  We have
  \begin{align*}
    N(-&g - G, \tau+1) - N(G, \tau) - g B(G, \tau)   \\
    &= \frac{1}{\tau+1}\bexp{\frac{-g - G}{\sqrt{\tau+1}}} 
    -\frac{1}{\tau}\bexp{\frac{G}{\sqrt{\tau}}}
    -\frac{g}{\tau^{3/2}} \bexp{\frac{G}{\sqrt{\tau}}}, \\
\intertext{and since this expression is increasing as $g$ decreases, 
  and $g \geq -1$ in any case,}
    &\leq \frac{1}{\tau+1}\bexp{\frac{1 - G}{\sqrt{\tau+1}}} 
    -\frac{1}{\tau}\bexp{\frac{G}{\sqrt{\tau}}}
    +\frac{1}{\tau^{3/2}} \bexp{\frac{G}{\sqrt{\tau}}}, \\
\intertext{and since $\tau^{3/2} > \tau$, taken together the second two terms increase 
  as $G$ decreases, as does the first term, so since $G \geq 0$,}
    &\leq \frac{1}{\tau+1}\bexp{\frac{1}{\sqrt{\tau + 1}}} 
    -\frac{1}{\tau}
    +\frac{1}{\tau^{3/2}} \quad = \quad \heps(\tau),
  \end{align*}
  and re-arranging proves the lemma.
\end{proof}

\begin{reptheorem}{thm:ftrl-bad}
\thmftrlbadtxt
\end{reptheorem}

\begin {proof}
  For simplicity, we will prove that regret is at least $\frac T 2
  (|\xs| - \eps_T)$ when $T$ is even; if $T$ is odd, we simply take
  $g_T = 0$ and consider the first $T-1$ rounds.

  Let $T = 2M$.  We will consider two gradient sequences.  First,
  suppose $g_t = 1$ for $t \le M$, and $g_t = -1$ otherwise.  Observe
  that for any $r$, we have $g_{1:M - r} = g_{1:M + r}$, which implies
  $x_{M - r + 1} = x_{M + r + 1}$.  Thus, the algorithm's total reward
  is
  \begin {align*}
    \sum_{t=1}^T x_t g_t
    & = \sum_{t=1}^{M} x_t - \sum_{t=M + 1}^T x_t \\
    & = x_1 - x_{M + 1}
    + \sum_{r=1}^{M - 1} x_{M - r + 1} - x_{M + r + 1} \\
    & = x_1 - x_{M + 1}
  \end {align*}
  Because $x_1 = 0$, we get that on this sequence the algorithm has
  origin-regret $\hat{x} \equiv x_{M + 1}$, and so by assumption $\hat{x} \leq
  \eps_T$.

  Next, suppose $g_1 = 1$ for $t \le M$, and $g_t = 0$ otherwise.  For
  this sequence, we will have $x_t \le \hat{x} \le \eps_T$ for all
  $t$, so total reward is at most $M \eps_T$.  For any positive $\xs$
  with $\xs > \eps_T$, this means that regret with respect to $\xs$ is
  at least
  \[
  \xs M - M \eps_T = M (|\xs| - \eps_T) \mbox { .}
  \]
  For $\xs < -\eps_T$, we can use a similar argument with the sign of
  the gradients reversed (for both gradient sequences) to get the same
  bound.
\end {proof}

In proving Theorem~\ref{thm:lower1d}, we will use the following lemma.
\begin{lemma}\label{lem:p}
Let $G_T = \sum_{i=1}^T g_i$ be the sum of $T$ random variables, each
drawn uniformly from $\{-1,1\}$.  Then, for any integer $k$ that is a
factor of $T$, we have
\[
\Prob[G_T \ge \sqrt{kT}] \ge p^k \mbox { .}
\]
where $p = \frac {7} {2^6} = 0.109375$.
\end {lemma}
\begin {proof}
First, for any $T$ define $p_T = \Prob[G_T \ge \sqrt{T}]$, and define
\[
p = \inf_{T \in \mathbb{Z}^+} p_T \mbox { .}
\]
For any $T$, we have $p_T \ge 2^{-T}$ trivially, and by the Central
Limit Theorem, $\lim_{T \rightarrow \infty} p_T = 1-\mathcal{N}_{0,
  1}(1) > 0$, where $\mathcal{N}_{0,1}$ is the standard normal
cumulative distribution function.  It follows that $p > 0$, and using
numerical methods we find $p = p_6 = \frac {7} {2^6} = 0.109375$.

Now, divide the length $T$ sequence into $k$ sequences of length
$\frac T k$.  Let $Z_i$ be the sum of gradients for the $i$th of these
sequences.  Observe that if $Z_i \ge \sqrt {\frac T k}$ for all $i$,
then $G_T = \sum_{i=1}^k Z_i \ge k \sqrt {\frac T k} = \sqrt {k T}$.
Furthermore, for any $i$, we have
\[
\Prob\left[Z_i \ge \sqrt{\frac T k}\right]
= \Prob\left[G_{\frac T k} \ge \sqrt {\frac T k}\right]
\ge p \mbox { .}
\]
Thus,
\[
\Prob\left[G \ge \sqrt {kT}\right]
\ge \prod_{i=1}^{k} \Prob\left[Z_i \ge \sqrt {\frac T k}\right]
\ge p^k \mbox { .}
\]
\end {proof}

\begin{reptheorem}{thm:lower1d}
\lemlowertext
\end {reptheorem}
\begin {proof}
Let $k = k(T) = \floor{\log(\frac {R \sqrt T} {\eps})/\log(p^{-1})}$, and
choose $T \ge T_0$ large enough so that $4 \le k \le T$ and also so
that $T$ is a multiple of $k$ (the latter is possible since $k(T)$
grows much more slowly than $T$).  Let $Q$ be the algorithm's
reward when each $g_t$ is drawn uniformly from $\{-1,1\}$.  Let $G =
g_{1:T}$.  As shown in the proof sketch, we have
\[
\E[Q | G \ge \sqrt{kT}] < \frac {\eps} {\Pr[G \ge \sqrt{kT}]} \mbox { .}
\]
By Lemma~\ref{lem:p}, $\Pr[G \ge \sqrt{kT}] \ge p^k$.  Thus,
\[
  \E[Q | G \ge \sqrt{kT}] < \eps p^{-k}
  = \eps \exp\p{k \log p^{-1}}
\le R \sqrt T \mbox { .}
\]
If the algorithm guaranteed $Q \ge R \sqrt T$ whenever $G \ge
\sqrt {kT}$, then we would have $\E[Q | G \ge \sqrt{kT}] \ge R \sqrt T$,
a contradiction.
Thus, there exists a sequence where $G \ge \sqrt{kT}$ and $Q < R\sqrt T$,
so on this sequence we have
\[
\Regret \ge R \sqrt {kT} - R \sqrt T = R \sqrt T (\sqrt k - 1)
\]
Because $k \ge 4$, we have $\frac 1 2 \sqrt k \ge 1$ or $\sqrt k - 1
\ge \frac 1 2 \sqrt k$, so regret is at least $\frac 1 2 R \sqrt
    {kT} = b R \sqrt {T \log\p{\frac {R \sqrt T} {\eps}}}$,
where $b = \frac 1 2 \sqrt {\frac {1} {\log p^{-1}}} > 0.336$ (and $p$
is the constant from Lemma~\ref{lem:p}).
\end {proof}
} %
\end{document}